\documentclass{article}

  \usepackage[preprint]{neurips_2019}

\usepackage[utf8]{inputenc} 
\usepackage[T1]{fontenc}    
\usepackage{hyperref}       
\usepackage{url}            
\usepackage{booktabs}       
\usepackage{amsfonts}       
\usepackage{nicefrac}       
\usepackage{microtype}      
\usepackage{graphicx}
\usepackage{float}
\usepackage{xcolor}

\usepackage{amsthm}
\usepackage{bm}
\usepackage{amsmath}
\usepackage{amssymb}

\theoremstyle{plain}

\newtheorem{theorem}{\bf Theorem}[section]

\newtheorem{assumption}{\bf Assumption}[section]

\def\r{\ensuremath\mathbb{R}}

\def\n{\ensuremath\mathbb{N}}

\newcommand{\rn}[1]{\mathbb{R}^{#1}}

\newcommand{\sq}[2]{ \{ #1 \}_{ #2 } }

\newcommand{\norm}[1]{ \Vert #1 \Vert}

\theoremstyle{remark}

\newcommand{\eqn}[2]{\begin{equation} \label{#1} #2 \end{equation}}
\newcommand{\eqnn}[1]{\begin{equation*} #1 \end{equation*}}

\newcommand{\CE}[2]{\ensuremath E[ #1 \ | \ #2 ]}

\newcommand{\CEs}[3]{\ensuremath E_{#1}[ #2 \ | \ #3 ]}

\title{A Convergence Result for Regularized Actor-Critic Methods}

\author{%
  Wesley Suttle \\
  \texttt{wesley.suttle@stonybrook.edu} \\
   \And
   Zhuoran Yang \\
   \texttt{zy6@princeton.edu} \\
   \AND
   Kaiqing Zhang \\
   \texttt{kzhang66@illinois.edu} \\
   \And
   Ji Liu \\
   \texttt{ji.liu@stonybrook.edu} \\
}

\begin{document}

\maketitle

\begin{abstract}

In this paper, we present a probability one convergence proof, under suitable conditions, of a certain class of actor-critic algorithms for finding approximate solutions to entropy-regularized MDPs using the machinery of stochastic approximation. To obtain this overall result, we prove the convergence of policy evaluation with general regularizers when using linear approximation architectures and show convergence of entropy-regularized policy improvement.

\end{abstract}

\section{Introduction}

Entropy-regularized methods are a promising class of reinforcement learning algorithms, but until recently firm theoretical underpinnings have been lacking. In an important step towards remedying this, \cite{geist} lays theoretical foundations for the treatment of approximation algorithms for a new class of regularized Markov decision processes (reg-MDPs), of which entropy-regularized MDPs are a special case. Though \cite{geist} establish the theory of reg-MDPs, provide a regularized policy gradient theorem, and provide an interesting and potentially useful analysis of error propagations for approximation algorithms for this class of problem, they stop short of providing stochastic convergence results for approximation schemes for reg-MDPs.

The use of regularization in Markov decision processes and reinforcement learning is not new. Regularization in exact and approximate MDPs has been explored in \cite{ziebart}, \cite{odonoghue}, \cite{nachum}, and \cite{geist}. Among the methods making use of the Kullback-Leibler divergence and/or entropy regularization in a fundamental way in this context are trust region policy optimization (TRPO) \cite{shulman15}, G-learning \cite{fox}, dynamic policy programming (DPP) \cite{azar}, Soft Q-learning \cite{haarnoja17}, soft actor-critic (SAC) \cite{haarnoja}, smoothed Bellman error embedding (SBEED) \cite{dai18}, and proximal policy optimization (PPO) \cite{shulman17_2}, to name a few. The connections between Soft Q-learning and policy gradients with a regularizer using the Kullback-Leibler divergence have been previously explored in \cite{shulman17}. 

To the best of our knowledge, however, no work yet exists that applies the machinery of stochastic approximation -- used to such great effect in the theory of unregularized reinforcement learning under function approximation -- to the regularized case. In this paper we take this next step. We first begin in section 2 with a description of the general setting of reg-MDPs, as described in \cite{geist}. Next, we introduce a two-timescale, actor-critic scheme for approximate reg-MDP, similar to those analyzed in \cite{bhatnagar09}, to which our convergence results will apply. In section 4, we prove that, for general reg-MDPs, policy evaluation (i.e., the critic updates) converges at the faster timescale with probability one to the global best approximator when using linear approximation architectures. In section 5, we prove under suitable conditions that the policy improvement (or actor) updates given in section 3 converge with probability one. We then provide some closing remarks.

\section{Regularized Markov Decision Processes}

\subsection{Markov Decision Processes}

In this section we introduce the general framework of reg-MDPs as developed in \cite{geist}. Before describing reg-MDPs, we first recall the definition of an unregularized, infinite-horizon, discounted Markov decision process (MDP). An MDP is a tuple $(\mathcal{S}, \mathcal{A}, P, r, \gamma)$, where $\mathcal{S}$ is the state space, $\mathcal{A}$ is the action space, $p : \mathcal{S} \times \mathcal{S} \times \mathcal{A} \rightarrow [0,1]$ is a family of conditional distributions over $\mathcal{S}$ conditioned on $\mathcal{S} \times \mathcal{A}$, i.e. $p(s' | s, a)$ is the probability that the next state will be $s'$, given that the current state is $s$ and action $a$ is chosen, $r : \mathcal{S} \times \mathcal{A} \rightarrow \r$ is the reward function, and $\gamma \in [0,1)$ is the discount factor. In this paper, we assume that $\mathcal{S}$ and $\mathcal{A}$ are finite.

In this setting, an agent interacts with the MDP via a policy $\pi : \mathcal{A} \times \mathcal{S} \rightarrow [0,1]$, where $\pi(a | s)$ is the probability that action $a$ is chosen in state $s$. The long-run discounted reward of following policy $\pi$ is given state-wise by $v_{\pi}(s) = \CE{\sum_{k=0}^{\infty} \gamma^k r(s_k, a_k)}{s_0 = s, \pi}$, where $v_{\pi} \in \rn{|\mathcal{S}|}$ is called the value function of $\pi$. The goal is to determine a policy that maximizes this return for each state. To accomplish this goal, classical MDP theory considers two operators on the space of value functions. Defining them component-wise, the first is the Bellman operator for $\pi$,
\eqn{bellman1}{[T_{\pi} v](s) = \sum_{a \in \mathcal{A}} \pi(a|s) ( r(s,a) + \gamma \sum_{s' \in \mathcal{S}} p(s'|s,a)  v(s')),}
and the second is the Bellman optimality operator
\eqn{bellman2}{[T v](s) = \max_{\pi} [T_{\pi} v] (s).}
%
%
Important in many applications is a related function $q_{\pi} \in \rn{|\mathcal{S}| \times |\mathcal{A}|}$, called the action-value function of $\pi$, given by
\eqnn{q_{\pi}(s,a) = r(s,a) + \gamma \sum_{s'} p(s'|s,a) v_{\pi}(s').}
Importantly, we can define a Bellman operator $T'_{\pi}$ on the space $\rn{|\mathcal{S}| \times |\mathcal{A}|}$ of action-value functions
\eqn{bellman3}{[T'_{\pi} q](s,a) = r(s,a) + \gamma \sum_{s'} p(s'|s,a) \sum_{a'} \pi(a'|s') q(s',a'),}
of which $q_{\pi}$ given above is the fixed point. Notice that this operator results from combining the definition of $q_{\pi}$ with the fact that $v_{\pi}$ is the fixed point of (\ref{bellman1}). We will be using action-value functions extensively in what follows.

\subsection{Regularized MDPs}

\cite{geist} introduce the concept of a regularized Markov decision process by adding a regularizing function $\Omega : \Delta_{\mathcal{A}} \rightarrow \r$ to the Bellman operator (\ref{bellman1}), where $\Delta_{\mathcal{A}}$ is the probability simplex in $\rn{|\mathcal{A}|}$, then using it in conjunction with the resulting regularized optimality operator to develop a regularized analog of classic MDP. The presence of $\Omega$ assigns a value to the conditional distributions that make up a given policy, allowing the user of any resulting algorithm to enforce a preference for some policies over others in the search for an optimal policy. The archetypal and primary motivating example of such a regularizer, as considered in \cite{haarnoja}, is $\Omega(\pi(\cdot | s)) = \sum_a \log (\pi(a|s)) \pi(a|s)$, the negative entropy of $\pi(\cdot | s)$, where the entropy $- \Omega$ is a commonly used measurement of the randomness of $\pi(\cdot | s)$. \cite{geist} develop their theory for a more general class of regularizers, however, requiring only that $\Omega$ be differentiable and strongly convex.

Given a policy $\pi$, its regularized Bellman operator is given by
\eqn{regbellman1}{[T_{\pi,\Omega} v](s) = [T_{\pi}v](s) - \Omega(\pi(\cdot | s)),}
while the regularized Bellman optimality operator is
\eqn{regbellman2}{[T_{\Omega} v](s) = \max_{\pi(\cdot | s) \in \Delta_{\mathcal{A}}} [T_{\pi,\Omega} v](s).}
As proven in \cite{geist}, the operators (\ref{regbellman1}), (\ref{regbellman2}) are also $\gamma$-contraction mappings, and we can thus immediately construct algorithms for policy evaluation and improvement in this regularized setting by iterating to obtain their fixed points. Focusing on (\ref{regbellman1}), the resulting fixed point $v_{\pi,\Omega}$ and associated action-value function $q_{\pi,\Omega}$ are given by
\eqn{regv}{v_{\pi,\Omega}(s) = \sum_a \pi(a | s) (q_{\pi,\Omega}(s,a)) - \Omega(\cdot | s),}
\eqn{regq}{q_{\pi,\Omega}(s,a) = r(s,a) + \gamma \sum_{s'} p(s' | s,a) v_{\pi,\Omega}(s').}
Similar expressions are available for the fixed point of (\ref{regbellman2}), but we will not need them below, and refer the reader to \cite{geist}.

It is also possible to define an analog of (\ref{bellman3}) for the regularized setting:
\eqn{regbellman3}{[T'_{\pi,\Omega} q](s,a) = r(s,a) + \gamma \sum_{s'} p(s'|s,a) \sum_{a'} \pi(a'|s') \big[q(s',a') - \Omega(\pi(a'|s')) \big].}
Since we will usually consider only (\ref{regbellman3}) and not (\ref{regbellman1}) in what follows, we drop the apostrophe and simply write $T_{\pi,\Omega} q$ in place of $T'_{\pi,\Omega} q$. Whether the state-value function or action-value function operator is meant will always be clear from the context.

We will henceforth also supress the dependence of our regularized action-value functions on the regularizer $\Omega$, and simply write $q_{\pi}$ in place of $q_{\pi,\Omega}$, and $v_{\pi}$ in place of $v_{\pi,\Omega}$.

\section{Actor-Critic Algorithm for Entropy-Regularized MDPs}

Given the superior performance of reinforcement learning methods such as soft actor-critic (SAC) \cite{haarnoja} that seek approximate solutions to a given MDP by using function approximatioon and entropy regularization to instead approximately solve a related reg-MDP, it is clear that reg-MDPs are of significant practical and theoretical importance. Following the development of classic actor-critic methods in reinforcement learning, a natural next step is to seek a firm theoretical foundation for actor-critic methods for approximate reg-MDP.

Below we give an entropy-regularized actor-critic algorithm in keeping with the classic actor-critic of unregularized reinforcement learning: it is a two-timescale stochastic approximation algorithm that performs gradient ascent in the policy parameters $\theta$ to maximize a specific measurement $J(\theta)$ of long-term reward. The algorithm we present takes its basic form from \cite{geist}. Let a finite state- and action-space reg-MDP as in section 2 be given. Let $q_{\omega}(\cdot, \cdot)$ be a family of action-value function approximators parametrized by $\omega \in U_{\omega} \subseteq \rn{k}$, and let $\pi_{\theta}(\cdot | \cdot)$ be a family of policy function parametrized by $\theta \in U_{\theta} \subseteq \rn{l}$, where $k, l \ll | \mathcal{S} |, | \mathcal{A} |$, and where $q_{\omega}$ and $\pi_{\theta}$ are continuously differentiable in their respective parameters. Let two positive stepsize sequences $\{ \beta_{\omega, t} \}, \{ \beta_{\theta, t} \}$ be given satisfying $\sum \beta_{\omega, t} = \infty, \sum \beta_{\theta, t} = \infty$, $\sum \beta_{\omega, t}^2  + \beta_{\theta, t}^2 < \infty$, and $\lim_t \frac{\beta_{\theta, t}}{\beta_{\omega, t}} = 0$. Finally, before proceeding, we need to make some standard assumptions regarding $\pi_{\theta}$. 
\begin{assumption}
There exists $\varepsilon > 0$ such that, for each $\theta \in U_{\theta}$, and for all state-action pairs $(s,a)$,  $\pi_{\theta}(s,a) \geq \varepsilon$.
\end{assumption}
\begin{assumption}
For each $\theta \in U_{\theta}$, the policy $\pi_{\theta}$ induces an ergodic, irreducible Markov chain on $\mathcal{S} \times \mathcal{A}$.
\end{assumption}
The first assumption ensures that all policies remain sufficiently exploratory, and is especially reasonable in light of the fact that entropy regularization penalizes overly deterministic policies. The second assumption is almost a consequence of the first, but it also imposes some additional structure on the transition probability function $p$. The objective of the algorithm is to maximize the expected, long-run regularized reward given by
$$J(\theta) = \sum_s d_{\theta}(s) \sum_a \pi_{\theta}(a | s) \big[ q_{\theta}(s,a) - \log (\pi_{\theta}(a|s)) \big],$$
where $d_{\theta}(s) = \sum_{t = 0}^{\infty} \gamma^t P(s_t = s \ | \ s_0 = s, \pi_{\theta})$ is the discounted weighting to states visited starting in $s$ and following $\pi_{\theta}$ \cite{sutton00}.
The entropy-regularized actor-critic algorithm that is the focus of this paper is given by the update equations
\eqn{regac1}{\omega_{t+1} = \omega_t + \beta_{\omega, t} \delta_{t+1} \nabla q_{\omega_t}(s_t, a_t),}
\eqn{regac2}{\theta_{t+1} = \theta_t + \beta_{\theta, t} \psi_t,}
where
\eqnn{\delta_{t+1} = r(s_t, a_t) + \gamma \big[ q_{\omega_t}(s_{t+1}, a_{t+1}) - \log ( \pi_{\theta_t} (a_t | s_t)) \pi_{\theta_t} (a_t | s_t) \big] - q_{\omega_t}(s_t, a_t)}
and
\eqnn{\psi_t = q_{\omega_t}(s_t, a_t) \nabla \log (\pi_{\theta_t}(a_t | s_t)) - \frac{1}{\pi_{\theta_t}(a_t | s_t)} \nabla \big[ \log ( \pi_{\theta_t} (a_t | s_t)) \pi_{\theta_t} (a_t | s_t) \big].}
Though the expressions for $\delta_{t+1}$ and $\psi_t$ may appear complicated at first sight, they should be natural to a reader familiar with temporal difference and actor-critic methods: $\delta_{t+1}$ is simply the regularized temporal difference, while  $\psi_t$ is the policy gradient derived in \cite{geist}, adapted to the entropy-regularized setting. It is important to note here that, when used in the algorithm above, the estimate of the policy gradient provided by $\psi_t$ is biased, in general. This is due to the fact that, as will be seen in the following section, the estimate of $q_{\omega_t}$ used in each update is itself a biased estimate of the true value function $q_{\theta_t}$ corresponding to the current policy $\pi_{\theta_t}$. We will discuss the effects of this fact on the behavior of the algorithm below.

Following the procedure in \cite{konda02}, assume that some probability distribution $\xi(\cdot)$ over the states $S$ has been given, and that the transitions of the MDP for the duration of the algorithm are generated as follows: with probability $\gamma$, the transition out of the current $(s,a)$ to next state $s'$ occurs according to $s' \sim p(\cdot | s, a)$; with probability $1-\gamma$, the next state is generated by $s' \sim \xi(\cdot)$. This ensures that, given a fixed policy parameter $\theta$, the steady-state distribution of the induced Markov chain is given by the $\gamma$-weighted occupancy measure given by $(1-\gamma) d_{\theta}$. When $\gamma$ is close to 1, this distribution is close to that of the Markov chain induced by $\pi_{\theta}$, and the transitions are in practice more or less safely generated as $s' \sim p(\cdot | s, a)$ at each timestep.

\section{Policy Evaluation for General Regularizers}

The algorithm above is formulated specifically with entropy as the regularizer of the underlying reg-MDP. For a fixed policy $\pi_{\theta}$, however, the almost sure convergence of the policy evaluation (critic) step is straightforward to prove for any regularizer, so we give the proof for this general case in the appendix. Proving it in this more general setting may also be helpful for future work on approximate reg-MDP.

The central result of this section is that, when using linear approximation architectures for $q_{\omega}$, and assuming the policy $\pi_{\theta}$ to be fixed, the update equation (\ref{regac1}) converges almost surely (a.s.) to the fixed point of a projected Bellman equation based on (\ref{regbellman3}). Intuitively, this means that (\ref{regac1}) converges to the global best approximator $q_{\omega^*}$ of $q_{\pi_{\theta}}$ when using linear approximation. Before we can state and prove this result formally, we need to develop some additional notation.

Consider the pairs of the state-action space $\mathcal{S} \times \mathcal{A}$ to be labeled and ordered as follows: $\mathcal{S} \times \mathcal{A} = \{ (s_1, a_1), (s_1, a_2), \ldots, (s_{|S|}, a_{|A|-1}), (s_{|S|}, a_{|A|}) \}.$ Let $\phi : \mathcal{S} \times \mathcal{A} \rightarrow \rn{K}$, where $K \ll | \mathcal{S} | \cdot | \mathcal{A} |$, be a mapping of the state-action space into the feature space $\rn{K}$. Let $\Phi = [ \phi(s_1, a_1), \phi(s_1, a_2), \ldots, \phi(s_{|S|}, a_{|A|}) ]^T$, the matrix whose rows are the feature vectors $\phi(s_i, a_j)$,\footnote{Unless explicitly stated otherwise, all vectors are assumed to be column vectors.} occuring from top to bottom in the order of the enumeration of $\mathcal{S} \times \mathcal{A}$ given above, and let $\bm{r} = [ r(s_1, a_1), r(s_1, a_2), \ldots, r(s_{|S|}, a_{|A|}) ]^T$ be the vector of rewards. We make the following standard assumption on the feature matrix $\Phi$.
\begin{assumption}
$\Phi$ has linearly independent columns.
\end{assumption}
For the remainder of this section, fix a regularizer $\Omega$ satisfying the conditions of section 2, and assume that $\pi := \pi_{\theta}$ is fixed. Since the $\omega$-updates of the critic step occur at the faster timescale of our algorithm, the current policy parameter $\theta$ appears to be fixed from $\omega$'s frame of reference. We can thus perform our $\omega$-updates as if the parameter $\theta$ and thus the policy $\pi$ is fixed. For details on two-timescale algorithms, see \cite{borkar}.

Let $\bm{q}_{\pi} = [q_{\pi}(s_1,a_1), q_{\pi}(s_1,a_2), \ldots, q_{\pi}(s_{|S|},a_{|A|}) ]^T$ be the vector of regularized action-values, let $\bm{q}_{\omega}$ be defined analogously as the vector of regularized action-value function approximator values, define $\Omega(\pi) = [ \Omega(\pi(\cdot | s_1)), \Omega(\pi(\cdot | s_2)), \ldots, \Omega(\pi(\cdot | s_{|S|}))]^T$, and let $\Omega_{\pi} = \Omega(\pi) \otimes \bm{1}$, where $\bm{1} \in \rn{|A|}$ is the vector of all ones and $\otimes$ denotes the Kronecker product. Let 
\[ P_{\pi} =
\begin{bmatrix}
	p(s_1 | s_1, a_1) \pi(a_1 | s_1) & p(s_1 | s_1, a_1) \pi(a_2 | s_1) & \ldots & p(s_{|S|} | s_1, a_1) \pi(a_{|A|} | s_1) \\
	p(s_1 | s_1, a_2) \pi(a_1 | s_1) & p(s_1 | s_1, a_2) \pi(a_2 | s_1) & \ldots & p(s_{|S|} | s_1, a_2) \pi(a_{|A|} | s_1) \\
	\vdots & \vdots & & \vdots \\
	p(s_{1} | s_{|S|}, a_{|A|}) \pi(a_1 | s_{|S|}) & p(s_1 | s_{|S|}, a_{|A|}) \pi(a_2 | s_{|S|}) & \ldots & p(s_{|S|} | s_{|S|}, a_{|A|}) \pi(a_{|A|} | s_{|S|}) \\
\end{bmatrix}
\]
denote the matrix of transition probabilities of the Markov chain it induces on $\mathcal{S} \times \mathcal{A}$, where $p(s_k | s_i, a_j) \pi(a_l | s_i)$ is the probability of transitioning from $(s_i, a_j)$ to $(s_k, a_l)$. Finally, let $\nu_{\pi}(\cdot,\cdot)$ denote the steady-state distribution of the Markov chain on $\mathcal{S} \times \mathcal{A}$ induced by $\pi$, and define $N_{\pi} = \text{diag}(\nu_{\pi}(s_1,a_1), \nu_{\pi}(s_1,a_2), \ldots, \nu_{\pi}(s_{|S|},a_{|A|}))$.

With the notation above, we can now write (\ref{regbellman3}) in matrix form as
\eqn{vecregbellman}{T_{\pi} q = \bm{r} + \gamma P_{\pi} (q - \Omega_{\pi}) = \bm{r} + \gamma P_{\pi} \Omega_{\pi} + \gamma P_{\pi} q.}
We now introduce the assumption of a linear approximation architecture for the action value function.
\begin{assumption}
Given parameter $\omega \in U_{\omega}$ and $(s,a) \in \mathcal{S} \times \mathcal{A}$, we have $q_{\omega}(s,a) = \phi(s,a)^T \omega$, or, in vector form, $\bm{q}_{\omega} = \Phi \omega$.
\end{assumption}
Under this assumption, the update equation (\ref{regac1}) becomes
\eqn{critic1}{\omega_{t+1} = \omega_t + \beta_{\omega, t} \delta_{t+1} \phi_t,}
where $\delta_{t+1} = r_{t+1} - \gamma \Omega(\pi(\cdot | s_{t+1})) + \gamma \phi_{t+1}^T \omega_t - \phi_t^T \omega_t$, $r_{t+1} = r(s_t,a_t)$, and $\phi_t = \phi(s_t,a_t)$. Notice that the exact evaluation $\Omega(\pi(\cdot | s_{t+1}))$ of $\Omega$ at $\pi(\cdot | s_{t+1})$ occurs in (\ref{critic1}), and an unbiased estimate of it occurs in (\ref{regac1}).

We now give the main result of this section.
\begin{theorem}
The iterative scheme (\ref{critic1}) converges a.s. to the unique $\omega^*$ such that $\Phi \omega^*$ is the fixed point of the projected Bellman equation
\eqn{criticpbe}{\Phi \omega = \Pi_{N_{\pi}} T_{\pi}(\Phi \omega),}
where $\Pi_{N_{\pi}}$ is the projection onto $\text{Col}(\Phi)$ with respect to the weighted Euclidean norm $||\cdot||_{2, N_{\pi}}$ and $T_{\pi}$ is as in (\ref{vecregbellman}).
\end{theorem}

\section{Policy Improvement with Entropy Regularization}

We now demonstrate the convergence of the actor updates (\ref{regac2}). To facilitate our convergence analysis, we make the following assumption.
\begin{assumption}
The set $U_{\theta}$ within which the policy parameters $\theta$ are constrained to lie is convex and compact, and the $\theta$-updates (\ref{regac2}) include a projection operator $\Gamma : \rn{l} \rightarrow U_{\theta}$, i.e.:
\eqn{actor1}{\theta_{t+1} = \Gamma \big( \theta_t + \beta_{\theta, t} \psi_t \big).}
\end{assumption}
The operator $\Gamma$ is usually taken to be the projection with respect to the Euclidean distance, but projections with respect to other metrics may be more useful in certain situations \cite{borkar}. It is worth noting that, since $U_{\theta}$ is convex, $\Gamma$ is single-valued. This projection procedure is a common technique in both the theory and practice of stochastic approximation: in theoretical work, it is a common way to stabilize the asymptotic behavior of sequences of iterates; in practice, we very often optimize over a specific feasible region. See \cite{borkar} and \cite{kushner} for further details on projected stochastic approximation, as well as \cite{bhatnagar09} for more details on the projection techniques applied to the actor step in the unregularized case.

Let $\mathcal{G} = \sigma (\theta_{\tau} ; \tau \leq t)$, and define
$$ \psi_{t, \theta_t} = q_{\omega_{\theta_t}}(s_t, a_t) \nabla \log (\pi_{\theta_t}(a_t | s_t)) - \frac{1}{\pi_{\theta_t} (a_t | s_t)} \nabla \big[ \log (\pi_{\theta_t}(a_t | s_t)) \pi_{\theta_t}(a_t|s_t) \big],$$
where $\omega_{\theta_t}$ is the limit point of the critic step for fixed $\theta_t$. Define 
$$h(\theta_t) = \CE{\psi_{t,\theta_t}}{\mathcal{G}_t} = \sum_{s_t \in \mathcal{S}} d_{\theta_t}(s_t) \sum_{a_t \in \mathcal{A}} \pi_{\theta_t}(s_t,a_t) \psi_{t, \theta_t},$$
and consider the following rephrasing of (\ref{actor1}):
\eqn{actor2}{\theta_{t+1} = \Gamma \big( \theta_t + \beta_{\theta,t} ( h(\theta_t) + \zeta_{t,1} + \zeta_{t,2}) \big),}
where $\zeta_{t,1} = \psi_t - \CE{\psi_t}{\mathcal{G}_t}$ and $\zeta_{t,2} = \CE{\psi_t - \psi_{t,\theta_t}}{\mathcal{G}_t}$. Consider also the ODE associated with (\ref{actor2}):
\eqn{actor3}{\dot{\theta} = \hat{\Gamma}(h(\theta)),}
where $\hat{\Gamma}$ is as defined in Appendix A.3 in the supplementary materials.
Notice that $h(\theta) = \nabla J(\theta)$. Since $J(\theta)$ is differentiable and $U_{\theta}$ is compact, the set of equilibria of (\ref{actor3}) consists of the stationary points of $h(\theta)$ on $U_{\theta}$.
We have the following.
\begin{theorem}
The sequence of iterates generated by (\ref{actor2}) converges a.s. to a stationary point of (\ref{actor3}).
\end{theorem}
\begin{proof}
We verify the conditions of the Kushner-Clark lemma given in the appendix.

To see that $h$ is continuous in $\theta_t$, we note that the stationary distribution $d_{\theta_t}$, as well as the policy $\pi_{\theta_t}$ are continuous in $\theta_t$, and that $\omega_{\theta_t}$ can also be shown to be continuous in $\theta_t$, since it is the unique solution (\ref{critic3}), and finally that $N_{\pi_{\theta_t}}, P_{\pi_{\theta_t}},$ and $\Omega_{\pi_{\theta_t}}$ are all continuous in $\theta_t$, where $\Omega(\pi_{\theta_t}(\cdot | s_t))$ is the negative entropy of $\pi_{\theta_t}$. Next, condition 1 of Theorem A.3 is satisfied by our assumptions on the stepsize sequence $\{ \beta_{\theta,t} \}$. For condition 3, we know by our proof of the critic step above that $\psi_t \rightarrow \psi_{t,\theta_t}$ a.s., and thus $\zeta_{t,2} \rightarrow 0$ a.s., and, since $\{ \omega_t \}$ is bounded a.s., we also have $\{ \zeta_{t,2} \}$ is bounded a.s.

All that's left to verify is condition 2. Define $\mathcal{M}_t = \sum_{\tau = 0}^t \beta_{\theta, \tau+1} \zeta_{\tau+1,1}.$ We then have
$$\CE{\mathcal{M}_{t+1}}{\mathcal{M}_t} = \CE{\beta_{\theta, t+1} \zeta_{t+1} + \mathcal{M}_t}{\mathcal{M}_t} = \beta_{\theta,t+1} \CE{\zeta_{t+1}}{\mathcal{M}_t} + \mathcal{M}_t = \mathcal{M}_t,$$
so $\{\mathcal{M}_t \}$ is a martingale. Notice that $\sum_{t=0}^{\infty} \norm{\beta_{\theta, t+1} \zeta_{t+1,1}}^2 < \infty$ a.s., since $\{\zeta_{t,1} \}$ is bounded a.s. by the fact that $\{\omega_t\}$ is bounded a.s. and since $\psi_t$ is continuous in $\theta_t$ and $U_{\theta}$ is compact. Thus
$$\sum_{t=0}^{\infty} \norm{\mathcal{M}_{t+1} - \mathcal{M}_t}^2 = \sum_{t=0}^{\infty} \norm{\beta_{\theta,t+1} \zeta_{t+1,1}}^2 < \infty$$
a.s., whence $\{ \mathcal{M}_t \}$ converges a.s. This implies that
$$\lim_t P \big( \sup_{n \geq t} \norm{\sum_{\tau=t}^{\infty} \beta_{\theta,\tau} \zeta_{\tau,1}} \geq \varepsilon \big) = 0,$$
completing the verification of condition 2 and thus the proof.
\end{proof}

As noted above, the fact that the $q$-function estimates given by the critic step at the faster timescale are in general biased affects the limit point $\omega^*$ of the algorithm. So long as the distance between $q_{\omega_t}$ and the true value $q_{\theta_t}$ is small, however, which is often the case when the approximator $q_{\omega}$ is sufficiently expressive, it can be shown that $\theta^*$ is nonetheless within a small neighborhood of a parameter corresponding to a locally optimal policy. It is technically possible for $\theta^*$ to correspond to a saddle point or minimum. Convergence to such ``unstable points'' can be prevented by suitable perturbation methods, however, and it is often the case in practice that the updates are sufficiently noisy to prevent such convergence from occurring. For more on non-convergence to unstable points, see \cite{kushner}.

\section{Conclusion}

In this paper we have provided fundamental stochastic convergence results for an important class of approximate solution methods to regularized Markov decision processes. Specifically, we have proved convergence of policy evaluation under linear function approximation for the case of general regularizers and demonstrated convergence of an entropy-regularized policy gradient method. We have contributed to the foundations for approximate reg-MDP in this paper, but many interesting and important open questions remain. Among the more theoretical future directions, extension of these results to continuous state and action spaces is key. Among the practical directions, significant empirical studies of the practical convergence properties, as well as comparison with state-of-the-art algorithms are warranted.

\newpage

\bibliographystyle{plain}
\bibliography{neurips19}

\newpage

\appendix
\section{Appendix}

\subsection{Proof of Theorem 4.1}

\begin{proof}
We first recast (\ref{critic1}) as an instance of the stochastic approximation scheme (\ref{sa:eq1}), then verify the stochastic approximation conditions given in section A.2. Define $h : \rn{\text{Rank}(\Phi)} \times \mathcal{S} \times \mathcal{A} \rightarrow \rn{\text{Rank}(\Phi)}$ by
\eqnn{h(\omega, s, a) = \CEs{s',a'}{(r(s,a) + \gamma (\phi(s',a')^T \omega - \Omega(\pi(\cdot | s'))) - \phi(s,a)^T \omega) \phi(s,a)}{\omega, s, a},}
and let $\xi_{t+1} = \delta_{t+1} \phi_t - \CE{\delta_{t+1} \phi_t}{\mathcal{F}_t}$, where $\mathcal{F}_t = \sigma(\omega_{\tau}, s_{\tau}, a_{\tau}; \tau \leq t)$ is the $\sigma$-algebra generated by the random variables $\omega_{\tau}, s_{\tau}, a_{\tau}$ up to time $t$. Notice that $\Phi \in \rn{|\mathcal{S}|\cdot|\mathcal{A}| \times K}$ has $\text{Rank}(\Phi) = K$, since we assumed $K \ll |\mathcal{S}|\cdot|\mathcal{A}|$ and that $\Phi$ has linearly independent columns. 

We can now rewrite (\ref{critic1}) as
\eqn{critic2}{\omega_{t+1} = \omega_t + \beta_{\omega,t} \big[ h(\omega_t, s_t, a_t) + \xi_{t+1} \big],}
and verify the five conditions of section A.2 to complete the proof.

Condition 2 is satisfied by the assumption on the stepsizes $\{\beta_{\omega,t}\}$ given in section 3, while condition 4 is ensured by assumption 3.2. To see that $h$ is Lipschitz in its first argument, fix $\omega, \omega', s, a$ and notice that
\[
\norm{h(\omega, s, a) - h(\omega', s, a)} = \norm{\CE{(\gamma \phi(s',a')^T - \phi(s,a)^T)(\omega - \omega') \phi(s,a)}{\omega, \omega', s, a}}
\]
$$\leq \CE{| (\gamma \phi(s',a')^T - \gamma \phi(s,a)^T)(\omega - \omega') | \cdot \norm{\phi(s,a)}}{\omega, \omega', s, a}$$
\[
\leq \CE{\norm{\gamma \phi(s',a') - \phi(s,a)}}{\omega, \omega', s, a} \cdot \norm{\phi(s,a)} \cdot \norm{\omega - \omega'} \leq C \norm{\omega - \omega'}
\]
for some $C > 0$, since the fact that $\mathcal{S}, \mathcal{A}$ are finite implies that $\norm{\gamma \phi(s',a') - \phi(s,a)}$ and $\norm{\phi(s,a)}$ are uniformly bounded, and where the first inequality follows by an application of Jensen's inequality and the second follows from the Cauchy-Schwarz inequality and the fact that the expectation is conditioned on $s, a, \omega,$ and $\omega'$. Thus condition 1 is satisfied.

We next show that $\{\xi_t\}$ is a martingale difference sequence with respect to the filtration $\{ \mathcal{F}_t \}$ such that $\CE{\norm{\xi_{t+1}}^2}{\mathcal{F}_t} \leq C(1 + \norm{\omega_t}^2)$ a.s., for some $C > 0$. First, we clearly have that
\[
\CE{\xi_{t+1}}{\mathcal{F}_t} = \CE{\delta_{t+1} \phi_t - \CE{\delta_{t+1} \phi_t}{\mathcal{F}_t}}{\mathcal{F}_t} = \CE{\delta_{t+1} \phi_t}{\mathcal{F}_t} - \CE{\delta_{t+1} \phi_t}{\mathcal{F}_t} = 0, 
\]
so $\{ \xi_t \}$ is indeed a martingale difference sequence. Next, by the Cauchy-Schwarz inequality we have that
\[
\norm{\xi_{t+1}}^2 \leq \norm{\delta_{t+1} \phi_t}^2 + 2 \norm{\delta_{t+1} \phi_t} \cdot \norm{\CE{\delta_{t+1} \phi_t}{\mathcal{F}_t}} + \norm{\CE{\delta_{t+1} \phi_t}{\mathcal{F}_t}}^2.
\]
Since $\phi_t$ is uniformly bounded on $\mathcal{S} \times \mathcal{A}$, we know that, for some $C_1 > 0$, and also by Jensen's inequality and another application of Cauchy-Schwarz, the right-hand side of the above inequality is
\[
\leq C_1 \Big[ |\delta_{t+1}|^2 + |\delta_{t+1}| \cdot \CE{|\delta_{t+1}|}{\mathcal{F}_t} + \CE{|\delta_{t+1}|^2}{\mathcal{F}_t} \Big].
\]
Now $|\delta_{t+1}| \leq |r_{t+1}| + \gamma | \Omega(\pi(\cdot|s_{t+1}))| + \norm{\gamma \phi_{t+1} - \phi_t} \cdot \norm{\omega_t},$ so, since $\Omega$ is continuous (since it is differentiable) and its domain is compact, and since $r(\cdot, \cdot)$ is uniformly bounded, we have $|\delta_{t+1}| \leq C_2 (1 + \norm{\omega_t}),$ for some $C_2 > 0$. This implies that there exists $C_3 > 0$ such that
\[
\norm{\xi_{t+1}}^2 \leq C_3 \Big[ (1 + \norm{\omega_t})^2 + (1 + \norm{\omega_t}) \CE{(1 + \norm{\omega_t})}{\mathcal{F}_t} + \CE{(1+ \norm{\omega_t})^2}{\mathcal{F}_t} \Big],
\]
which in turn implies the existence of some $C_4 > 0$ such that $\norm{\xi_{t+1}}^2 \leq C_4 ( 1 + \norm{\omega_t})^2$, and thus, for some $C > 0$ large enough, that $\norm{\xi_{t+1}}^2 \leq C(1 + \norm{\omega_t}^2)$ for all possible values of $\xi_{t+1}$. We thus trivially have that $\CE{\norm{\xi_{t+1}}^2}{\mathcal{F}_t} \leq C(1 + \norm{\omega_t}^2)$ a.s., satisfying condition 3.

Taking a closer look at the definition of $h$, we have
\[
h(\omega, s, a) = \CEs{s',a'}{\big[ r(s,a) + \gamma (\phi(s',a')^T \omega - \Omega(\pi(\cdot | s'))) - \phi(s,a)^T \omega \big] \phi(s,a)}{\omega, s, a}
\]
\[
= \sum_{s', a'} p(s'|s,a) \pi(a'|s') \big[ r(s,a) + \gamma (\phi(s',a')^T \omega - \Omega(\pi(\cdot | s'))) - \phi(s,a)^T \omega \big] \phi(s,a).
\]
Thus
\[
\overline{h}(\omega) = \sum_{s,a} \phi(s,a) \nu_{\pi}(s,a) \sum_{s', a'} p(s'|s,a) \pi(a'|s') \big[ r(s,a) + \gamma ( \phi(s',a')^T \omega - \Omega(\pi(\cdot | s')) ) - \phi(s,a)^T \omega \big]
\]
\[
= \sum_{s,a} \phi(s,a) \nu_{\pi}(s,a) \Big[ r(s,a) + \gamma \sum_{s', a'} p(s'|s,a) \pi(a'|s') \big[ \phi(s',a')^T \omega - \Omega(\pi(\cdot | s')) \big] - \phi(s,a)^T \omega \Big]
\]
\[
= \Phi^T N_{\pi} \Big( \bm{r} + \gamma P_{\pi} \big( \Phi \omega - \Omega_{\pi} \big) - \Phi \omega \Big).
\]
Considering the ODE $\dot{\omega} = \overline{h}(\omega)$, we see that it has an equilibrium point when $\overline{h}(\omega) = 0$, i.e. when
\[
\Phi^T N_{\pi} \Big( \bm{r} + \gamma P_{\pi} \big( \Phi \omega - \Omega_{\pi} \big) - \Phi \omega \Big) = \Phi^T N_{\pi}(\bm{r} - \gamma P_{\pi} \Omega_{\pi} + \gamma P_{\pi} \Phi \omega - \Phi \omega)
\]
\[
= \Phi^T N_{\pi} (\bm{r} - \gamma P_{\pi} \Omega_{\pi}) - \Phi^T N_{\pi} (I - \gamma P_{\pi}) \Phi \omega = 0,
\]
i.e. 
\[
\Phi^T N_{\pi} (I - \gamma P_{\pi}) \Phi \omega = \Phi^T N_{\pi} ( \bm{r} - \gamma P_{\pi} \Omega_{\pi} ).
\]
Since $\Phi^T N_{\pi} (I - \gamma P_{\pi}) \Phi$ is invertible, this equilibrium point
\eqn{critic3}{\omega^* = (\Phi^T N_{\pi} (I - \gamma P_{\pi}) \Phi)^{-1} (\Phi^T N_{\pi} ( \bm{r} - \gamma P_{\pi} \Omega_{\pi}))}
is unique, which completes verification of condition 5.

The resulting $\Phi \omega^*$ is precisely the fixed point of the projected Bellman equation (\ref{criticpbe}), \cite{bertsekas}. So, if we can verify the sufficient conditions of Theorem A.2, we can finally invoke Theorem A.1 to complete the proof. By the definition given in Theorem A.2,
\[
h_{\infty}(\omega) = \lim_{c \rightarrow \infty} c^{-1} \overline{h}(c \omega) = \lim_{c \rightarrow \infty} c^{-1} \Phi^T N_{\pi} \Big( \bm{r} + \gamma P_{\pi} \big( \Phi c \omega - \Omega_{\pi} \big) - \Phi c \omega \Big) = \Phi^T N_{\pi} ( \gamma P_{\pi} - I ) \Phi \omega.
\]
The ODE $\dot{\omega} = h_{\infty}(\omega)$ clearly exists and has 0 as its unique globally asymptotically stable equilibrium, whence $\sup_t \norm{\omega_t} < \infty$ a.s., and we invoke Theorem A.1 to obtain that $\omega_t \rightarrow \omega^*$ a.s., completing the proof.

\end{proof}

\subsection{Stochastic Approximation Conditions} Much of the theory concerning reinforcement learning under function approximation, as well as the current work, relies on the following key results of stochastic approximation taken from \cite{borkar}. The form in which the following is presented follows \cite{zhang} quite closely.

Consider the stochastic approximation scheme in $\rn{k}$ given by the update equation
\eqn{sa:eq1}{x_{n+1} = x_n + \alpha_n [h(x_n, Y_n) + \mathcal{M}_{n+1}],}
where $n \in \n$ and $x_0$ is given. Consider also the following conditions.

\begin{enumerate}
\item $h : \rn{k} \times \mathcal{B} \rightarrow \rn{k}$ is Lipschitz continuous in its first argument $x \in \rn{k}$.
\item $\sq{\alpha_n}{n \in \n}$ satisfies $\sum_n \alpha_n = \infty$, $\sum_n \alpha_n^2 < \infty$, and $\alpha_n \geq 0$ for all $n \in \n$.
\item $\sq{\mathcal{M}_n}{n \in \n}$ is a martingale difference sequence with respect to the filtration given by $\mathcal{F}_n = \sigma(x_m, \mathcal{M}_m; m \leq n) = \sigma(x_0, \mathcal{M}_m; m \leq n)$, and furthermore
\eqn{sa:eq2}{\CE{\norm{\mathcal{M}_{n+1}}^2}{\mathcal{F}_n} \leq K(1 + \norm{x_n}^2) \text{ a.s.},}
for all $n \in \n$.
\item $\{Y_n\}$ is an ergodic, irreducible Markov chain on the finite set $\mathcal{B}$ with stationary distribution $\eta$.
\item The ODE
\eqn{sa:eq3}{\dot{x} = \overline{h}(x) = \sum_{b \in \mathcal{B}} \eta(b) h(x, b)}
has a unique globally asymptotically stable equilibrium $x^*$.
\end{enumerate}
Under conditions 1-5 above, we have Theorem A.1 below. Under conditions 1-4 only, we have Theorem A.2.
\begin{theorem}
If $\sup_t \norm{x_t} \leq \infty$ a.s., then $x_t \rightarrow x^*$ a.s.
\end{theorem}
\begin{theorem}
If the function
\eqnn{h_{\infty}(x) = \lim_{c \rightarrow \infty} c^{-1} \overline{h}(cx)}
exists uniformly on compact sets for some $h_{\infty} \in C(\rn{n})$, then, if the ODE $\dot{y} = h_{\infty}(y)$ has the origin as its unique globally asymptotically stable equilibrium, we have $\sup_t \norm{x_t} < \infty$ a.s.
\end{theorem}

\subsection{Kushner-Clark Lemma} The above convergence proof for the actor step relies on the following result \cite{kushner}, \cite{prasad} for projected stochastic approximation schemes.

Let $\Gamma : \rn{m} \rightarrow \rn{m}$ be a projection onto a compact, convex set $K \subset \rn{m}$. Let
\eqnn{\hat{\Gamma}(h(x)) = \lim_{\epsilon \downarrow 0} \frac{\Gamma(x + \epsilon h(x)) - x}{\epsilon},}
for $x \in K$, and assume $h : \rn{m} \rightarrow \rn{m}$ is continuous on $K$. Consider the update
\eqn{kclemma:eq1}{x_{t+1} = \Gamma(x_t + \alpha_t (h(x_t) + \zeta_{t,1} + \zeta_{t,2}))}
and its associated ODE
\eqn{kclemma:eq2}{\dot{x} = \hat{\Gamma}(h(x)).}
\begin{theorem}
Under the following assumptions, if (\ref{kclemma:eq2}) has a compact set $K'$ as its asymptotically stable equilibria, then the updates (\ref{kclemma:eq1}) converge a.s. to $K'$.
\end{theorem}

\begin{enumerate}
\item $\sq{\alpha_t}{t \in \n}$ satisfies $\sum_t \alpha_t = \infty, \sum_t \alpha^2_t < \infty.$
\item $\sq{\zeta_{t,1}}{t \in \n}$ is such that
\eqnn{\lim_t P \Big( \sup_{n \geq t} \norm{\sum_{\tau = t}^n \alpha_{\tau} \zeta_{\tau,1}} \geq \epsilon \Big) = 0,}
for all $\epsilon > 0$.
\item $\sq{\zeta_{t,2}}{t \in \n}$ is an a.s. bounded random sequence with $\zeta_{t,2} \rightarrow 0$ a.s.
\end{enumerate}

\end{document}